\documentclass[11pt]{article}

\usepackage[margin=1in]{geometry}
\usepackage{amsmath,amssymb,amsthm,amsfonts,mathtools,bm}
\usepackage{graphicx}
\usepackage{booktabs}
\usepackage{hyperref}
\usepackage{enumitem}
\usepackage{microtype}
\usepackage{mathrsfs}
\usepackage{bbm}
\usepackage{float}
\usepackage{placeins}

\hypersetup{colorlinks=true, linkcolor=blue, citecolor=blue, urlcolor=blue}

\newtheorem{theorem}{Theorem}
\newtheorem{proposition}{Proposition}

\newtheorem{corollary}{Corollary}

\newtheorem{assumption}{Assumption}

\title{Compressed-Sensing-Guided, Inference-Aware Structured Reduction for Large Language Models}
\author{
Andrew J.\ Kiruluta\\
{\small UC Berkeley School of Information}}
\date{}

\begin{document}
\maketitle

\begin{abstract}
Large language models achieve strong generative performance at the cost of extreme parameter counts, large memory footprints, and substantial decoding latency. Existing compression methods have shown that sparse or structured reductions can preserve accuracy under aggressive pruning, and prompt-compression methods have further demonstrated that substantial latency reductions can be obtained by removing redundant input content before inference. However, these two lines of work remain largely disjoint. Most model-compression approaches are static, are optimized offline, and do not explicitly exploit the fact that different prompts and even different decoding steps activate different latent computational pathways. Conversely, prompt-compression methods reduce sequence length but do not directly adapt the executed model subnetwork. In this manuscript we propose a unified compressed-sensing-guided framework in which random measurement operators are used to probe latent model usage, sparse recovery is used to estimate task-conditioned and token-adaptive support sets, and the recovered supports are compiled into hardware-efficient sparse execution paths over blocks, attention heads, channels, and feed-forward substructures.

The proposed formulation contributes five coupled novelties. First, it introduces task-conditioned measurements, so that different prompts induce different sparse supports and therefore different computational graphs. Second, it enables token-adaptive recovery, so that active substructures are re-estimated during decoding rather than fixed once offline. Third, it provides a formal sample-complexity analysis showing how many probe measurements are required to recover the active support under restricted isometry or mutual incoherence assumptions. Fourth, it imposes compile-to-hardware constraints so that recovered supports are restricted to structures compatible with efficient GPU kernels and practical runtime speedup. Fifth, it unifies prompt compression and model reduction in a single compressed-sensing objective, thereby coupling input-token selection and subnetwork selection rather than optimizing them independently. The resulting framework defines dynamic large language model execution as a measurement-and-recovery problem, with explicit approximation guarantees and deployment-oriented constraints.
\end{abstract}

\section{Introduction}

The modern large language model is simultaneously a triumph of scale and a systems bottleneck. Increasing model size has improved reasoning, coding, and instruction-following performance, yet the associated computational costs have made deployment increasingly expensive in memory, bandwidth, energy, and end-to-end latency. As a consequence, a major body of work has investigated pruning, structured compression, sparse execution, and prompt compression as complementary routes toward more efficient inference. One-shot pruning methods such as SparseGPT have shown that large autoregressive transformers can be pruned at substantial sparsity without retraining, and activation-aware methods such as Wanda have shown that weight importance can be estimated effectively from weight--activation interactions. In parallel, inference-aware structured pruning methods such as ZipLM have highlighted the practical distinction between removing parameters and obtaining actual speedups on target hardware. Prompt-compression methods such as LLMLingua and LongLLMLingua have additionally demonstrated that many prompt tokens are redundant for downstream generation and can be removed to reduce cost while preserving utility. These advances establish that both model structure and input context contain substantial redundancy, but they stop short of treating inference itself as a dynamic sparse recovery problem in which the active computational pathway depends on the prompt and evolves across generated tokens.

This observation motivates a different perspective. Rather than beginning from the premise that the full dense network must always be executed and then asking which components may be removed statically, we begin from the premise that, for a particular prompt and a particular decoding step, only a sparse subset of latent computational structures is truly necessary. The challenge is then not merely to prune weights, but to \emph{identify} the active support of blocks, heads, channels, or neurons from a small number of cheap measurements, and to do so quickly enough that the measurement process itself does not dominate inference. This viewpoint is naturally aligned with compressed sensing. In compressed sensing, one does not directly observe a high-dimensional signal in full; instead, one acquires a limited number of linear or approximately linear measurements and then recovers the sparse latent representation under structural assumptions. Translating this idea to large language models suggests a dynamic architecture in which random feature sketches, calibration probes, or low-dimensional activation measurements are used to infer which structured subcomponents should execute for the current task or token.

The central thesis of this manuscript is that compressed sensing provides not only a metaphor but also a mathematical and algorithmic framework for dynamic large language model execution. The proposed method uses randomized measurements of prompt-conditioned and token-conditioned features, recovers sparse support sets over structured computational units, and compiles those support sets into hardware-efficient sparse kernels. In contrast to purely static pruning, the support is allowed to vary across prompts. In contrast to conventional routing heuristics, the selection of active substructures is formulated as a sparse inverse problem with explicit assumptions and guarantees. In contrast to model-only compression, the framework also incorporates prompt compression in the same recovery objective, thereby jointly optimizing what information enters the model and which part of the model processes it.

The main novelty of the framework lies in the interaction among five components. First, the measurement operators are task-conditioned, capturing the empirical fact that different prompts induce different internal usage patterns. Second, the sparse support can be recovered online during decoding, yielding token-adaptive execution rather than a single static reduced model. Third, the number of measurements required for reliable support recovery can be bounded using standard compressed-sensing tools such as restricted isometry properties and mutual incoherence conditions. Fourth, the admissible supports are restricted to structured sparsity patterns that map to efficient kernels on GPUs or similar accelerators. Fifth, prompt compression and model subnetwork recovery are solved jointly, so the system learns when it is preferable to shorten the sequence, when it is preferable to reduce the executed model, and when both forms of compression should occur together.

\section{Relationship to Prior Work and Positioning}

Large language model pruning is already a mature and rapidly developing area, and any new framework must therefore be positioned carefully. SparseGPT demonstrated that one-shot pruning of very large autoregressive transformers can be posed as a sparse regression problem and solved layerwise with high empirical fidelity. Wanda showed that activation information can be incorporated in a particularly simple form by weighting magnitudes with input activations, leading to strong pruning performance without retraining. ZipLM then sharpened the systems perspective by emphasizing that unstructured sparsity often fails to translate into real speedups and by introducing inference-aware structured pruning targeted to specific runtime environments. These works establish that sparsity is both algorithmically meaningful and practically valuable, but their primary logic is still offline model reduction: they estimate importance, remove parameters, and then deploy a reduced but largely static network.

Prompt compression follows a parallel but distinct trajectory. LLMLingua and its extensions show that prompts often contain substantial redundancy and that removing unimportant tokens can reduce cost while preserving semantic integrity and task performance. Yet prompt compression does not typically induce a corresponding dynamic reduction in the executed subnetwork. The model remains dense, and the reduced input merely changes its length. From a systems perspective, prompt compression and model compression therefore attack different terms in the total inference cost, but they are rarely optimized in a coupled fashion.

The proposed work differs from these lines in a principled way. It does not merely prune weights, compress prompts, or heuristically route tokens. Instead, it formulates prompt-conditioned and token-conditioned execution as structured support recovery from random measurements. The emphasis is not simply that sparsity exists, but that the sparse support can be inferred from a small sketch of model-state information and that this support can be constrained to structures that yield actual hardware acceleration. The novelty is therefore not the generic claim that language models are compressible, which is already well established, but the specific synthesis of compressed-sensing measurement design, dynamic support recovery, structured compile-time constraints, and joint prompt--model compression.

\section{Problem Formulation}

Consider an autoregressive transformer with $L$ layers. At decoding step $t$, let $x_{1:t}$ denote the visible token sequence, consisting of the compressed or uncompressed prompt together with generated tokens up to time $t$. Let $h_t^{(\ell)} \in \mathbb{R}^{d_\ell}$ denote the hidden state entering layer $\ell$ at step $t$. We assume that each layer is decomposed into a collection of structured computational units. These units may represent entire transformer blocks, attention heads, channel groups, feed-forward subblocks, or hardware-aligned tiles. Let the set of all such units be indexed by $\mathcal{G} = \{1,\dots,G\}$. For each layer and step, only a subset of these units is expected to be materially active for the current task.

We model this by introducing a binary or relaxed support vector
\[
s_t \in \{0,1\}^{G}
\quad \text{or} \quad
s_t \in [0,1]^G,
\]
where $s_{t,g}=1$ indicates that structured unit $g$ is executed at step $t$. The effective model at time $t$ is therefore a subnetwork selected by $s_t$. If $\mathcal{F}(x_{1:t};\theta)$ denotes the full dense model with parameters $\theta$, then the dynamically reduced model is written
\[
\mathcal{F}(x_{1:t};\theta,s_t),
\]
where masked or compiled operators enforce execution only on the support indicated by $s_t$.

The key difficulty is that $s_t$ is not known a priori. Computing it by exhaustively evaluating all possible subnetworks would defeat the purpose of compression. We therefore introduce a \emph{measurement operator} that observes a low-dimensional sketch of the prompt-conditioned and token-conditioned latent state. Let
\[
z_t = A_t \,\phi(x_{1:t}, h_t^{(1)}, \dots, h_t^{(L)}) + \varepsilon_t,
\]
where $\phi$ is a feature map constructed from lightweight probes, cached statistics, or partial forward information, $A_t \in \mathbb{R}^{m_t \times D}$ is a random or pseudo-random measurement matrix, and $\varepsilon_t$ represents modeling and measurement noise. The objective is then to recover $s_t$ from $z_t$ under a sparsity prior.

To connect this to compressed sensing, we introduce a structured dictionary $\Psi \in \mathbb{R}^{D \times G}$ such that the latent feature map can be approximated by
\[
\phi(x_{1:t}, h_t^{(1)}, \dots, h_t^{(L)}) \approx \Psi \alpha_t,
\]
where $\alpha_t \in \mathbb{R}^G$ is sparse and its support corresponds to the active substructures. The measurement equation becomes
\[
z_t = A_t \Psi \alpha_t + \varepsilon_t.
\]
Recovery of the active support may then be posed as
\[
\hat{\alpha}_t
=
\arg\min_{\alpha \in \mathbb{R}^G}
\frac{1}{2}\|z_t - A_t \Psi \alpha\|_2^2
+
\lambda \|\alpha\|_{1,\mathcal{S}}
+
\Omega_{\mathrm{hw}}(\alpha)
+
\Omega_{\mathrm{temp}}(\alpha,\hat{\alpha}_{t-1}),
\]
where $\|\cdot\|_{1,\mathcal{S}}$ is a structured sparsity norm, $\Omega_{\mathrm{hw}}$ enforces compile-to-hardware constraints, and $\Omega_{\mathrm{temp}}$ regularizes temporal variation across decoding steps. The recovered support is then
\[
\hat{s}_t = \mathrm{supp}_{\tau}(\hat{\alpha}_t),
\]
where $\mathrm{supp}_{\tau}$ thresholds or projects $\hat{\alpha}_t$ into an admissible structured support.

\section{Task-Conditioned Measurements}

A central empirical premise of the proposed framework is that different prompts activate different latent computational pathways. A mathematical model that assumes a single universal support for all prompts is therefore unnecessarily restrictive and likely suboptimal. To reflect this, we introduce task-conditioned measurements in which the measurement design and the recovered support are functions of the prompt distribution.

Let $p$ denote a prompt, or more generally a task-context pair. We write the latent support as $s^\star(p,t)$ to emphasize that the active subnetwork depends both on the prompt and on the decoding step. The measurement operator may also depend on the prompt through a lightweight controller:
\[
A_t(p) = \Gamma(\eta(p,t)),
\]
where $\eta$ is a compact prompt encoder and $\Gamma$ maps the resulting summary into a measurement ensemble. This does not require expensive optimization over $A_t$; in practice, one may choose from a finite bank of random sketches or structured projections according to coarse task attributes inferred from the prompt. The purpose of task-conditioning is not to overfit measurements to individual sequences, but to exploit the fact that prompts belonging to different semantic or computational regimes often exhibit distinct support patterns.

The theoretical advantage of task-conditioned measurements is that they improve the alignment between the measurement operator and the sparse structure to be recovered. If the family of supports is heterogeneous across tasks, then a universal measurement design must succeed on the union of many support classes, increasing sample complexity. By contrast, if the prompt narrows the feasible support family from a large ambient class $\mathcal{C}$ to a smaller class $\mathcal{C}(p)$, then one expects improved recovery at fixed measurement budget. This observation will later be reflected in a prompt-conditional sample complexity bound in which the effective complexity depends on the localized support family rather than the global model.

Operationally, task-conditioned measurements also create a mechanism for specialization without duplicating the model. Instead of maintaining many task-specific dense checkpoints, the framework retains a common parameter pool but dynamically recovers different sparse execution paths. A coding-oriented prompt, for example, may activate different head patterns and feed-forward channels than a summarization or mathematical reasoning prompt. The resulting model behavior is therefore both compressed and specialized, yet it remains derived from a single shared backbone.

\section{Token-Adaptive Recovery and Dynamic Execution}

Static compression assumes that a single reduced subnetwork suffices for an entire prompt-response episode. While such a reduction can be useful, it ignores the fact that autoregressive decoding is inherently nonstationary. Early decoding steps may rely strongly on prompt-parsing heads and retrieval-oriented channels, while later steps may shift emphasis toward syntactic closure, numerical consistency, or long-range dependency maintenance. A fixed support therefore risks either over-provisioning computation or under-provisioning it at critical moments.

For this reason we propose token-adaptive recovery. At each decoding step $t$, the model forms a low-cost sketch $z_t$ of the current latent state and solves a structured sparse recovery problem to estimate $\hat{s}_t$. The reduced computation for the next step then uses only those units in $\hat{s}_t$, together with minimal auxiliary routing overhead. The support may be updated every step, every few steps, or event-triggered when uncertainty exceeds a threshold.

This dynamic scheme can be understood as a control loop embedded within autoregressive inference. The plant is the large language model, the latent state is its hidden representation, the sensor is the measurement operator, the estimator is the sparse recovery solver, and the actuator is the compiled sparse execution kernel. The recovery problem is regularized temporally through a term such as
\[
\Omega_{\mathrm{temp}}(\alpha_t,\alpha_{t-1})
=
\gamma \|\alpha_t - \alpha_{t-1}\|_2^2
\]
or a support-change penalty. This discourages unnecessary oscillation in the executed subnetwork while still permitting adaptation when the token-level computational regime genuinely changes.

The token-adaptive formulation offers two advantages. First, it provides additional sparsity by exploiting temporal locality: if only a few units change from one decoding step to the next, then the recovery problem becomes easier and the active support can be updated incrementally. Second, it preserves quality more effectively than overly aggressive static pruning, because the model is allowed to recruit additional computation on difficult tokens while remaining sparse on easy ones. In this sense, dynamic support recovery is a form of conditional computation grounded in a compressed-sensing inverse problem rather than in a purely learned gating heuristic.

\section{Joint Prompt and Model Compression}

Sequence length and model size contribute differently to inference cost, but they interact multiplicatively in the total number of operations. It is therefore suboptimal to compress them independently whenever a joint formulation is available. To address this, we introduce a coupled decision variable consisting of a prompt-retention vector $r \in \{0,1\}^n$ over the input tokens and a model-support vector $s_t$ over structured computational units. The retained prompt is
\[
\tilde{x}_{1:n} = \Pi(r) x_{1:n},
\]
where $\Pi(r)$ denotes the token-selection operator. The dynamic inference problem then becomes
\[
\min_{r,\{\alpha_t\}}
\mathcal{L}_{\mathrm{task}}\big(\tilde{x}_{1:n}, \{\alpha_t\}\big)
+ \lambda_p \|r\|_0
+ \lambda_m \sum_{t=1}^T \|\alpha_t\|_{1,\mathcal{S}}
+ \Omega_{\mathrm{hw}}(\{\alpha_t\})
+ \Omega_{\mathrm{faith}}(r),
\]
subject to the measurement equations
\[
z_t = A_t(r)\Psi(r)\alpha_t + \varepsilon_t.
\]

This objective reveals an important conceptual shift. Prompt compression is no longer merely a preprocessing step. Instead, removing a token changes the latent measurements and therefore changes the recovered active support. Conversely, choosing a more expressive support may reduce the need to retain certain prompt tokens. The framework thus allocates a finite inference budget across two complementary resources: tokens and subnetworks. A longer prompt may be processed by a smaller active model if the retained tokens are highly informative; alternatively, a more aggressive prompt reduction may be compensated by a richer executed subnetwork.

The faithfulness term $\Omega_{\mathrm{faith}}(r)$ constrains prompt compression so that semantic or task-critical information is preserved. This term may instantiate KL divergence to a teacher distribution, answer consistency penalties, or token-importance constraints derived from a prompt compressor. The novelty is not to replace prompt compression methods such as LLMLingua, but to incorporate their logic into a joint compressed-sensing recovery problem in which prompt retention influences the measurement geometry and ultimately the selected execution path.

\section{Hardware-Aware Structured Recovery}

Model sparsity is only valuable in deployment if it aligns with the capabilities of the target hardware. Unstructured sparsity often reduces parameter count without delivering proportional latency gains because the resulting memory accesses and control flow remain inefficient. For this reason the proposed framework restricts support recovery to hardware-admissible structures. These may include block sparsity in matrix multiplication tiles, head sparsity aligned with attention-kernel batching, channel sparsity aligned with tensor-core dimensions, or semi-structured patterns such as $N\!:\!M$ masks.

Formally, let $\mathcal{H}$ denote the family of supports that can be compiled into efficient kernels on a target accelerator. Then the recovered support must satisfy
\[
\hat{s}_t \in \mathcal{H}.
\]
Equivalently, the coefficient vector $\alpha_t$ is projected onto a structured feasible set:
\[
\hat{\alpha}_t
=
\arg\min_{\alpha}
\frac{1}{2}\|z_t - A_t \Psi \alpha\|_2^2 + \lambda \|\alpha\|_{1,\mathcal{S}}
\quad \text{subject to } \mathrm{supp}(\alpha)\in\mathcal{H}.
\]
This turns the sparse recovery problem into a constrained structured estimation problem. Although this may appear more restrictive than classical compressed sensing, it is essential for practical deployment because it avoids supports that are mathematically sparse but operationally useless.

The compile-to-hardware constraint also alters the design of the measurement process. If recovery is only needed over the feasible class $\mathcal{H}$, then the effective search space is smaller than the set of all sparse supports. This can reduce measurement complexity, especially when $\mathcal{H}$ has low combinatorial complexity. At the same time, it makes the recovery guarantee more meaningful from a systems standpoint, because successful support recovery now implies executable acceleration rather than merely abstract sparsity.

Finally, hardware-aware recovery allows offline compilation and online dispatch to coexist. A finite library of sparse kernels corresponding to feasible support motifs can be precompiled. Online recovery then selects among these kernels or composes a small number of them, reducing runtime overhead. In effect, the compressed-sensing estimator determines which precompiled execution path to activate for the current token.

\section{Measurement Design and Sparse Recovery}

To make the framework concrete, we now specify the measurement and recovery pipeline. For each decoding step, the model first computes a lightweight feature vector
\[
u_t = \phi(x_{1:t}, h_t^{(1)}, \dots, h_t^{(L)}),
\]
where $\phi$ may include projected token embeddings, low-rank hidden summaries, head-score sketches, norm statistics, or cache-derived signals. The feature map must be substantially cheaper than a full dense forward pass and should be constructed from quantities already available in inference whenever possible.

A random measurement matrix $A_t \in \mathbb{R}^{m_t \times D}$ then produces
\[
z_t = A_t u_t + \varepsilon_t.
\]
The choice of $A_t$ may follow standard compressed-sensing ensembles such as subgaussian or randomized orthogonal projections, or structured transforms such as subsampled randomized Hadamard operators when implementation efficiency is important. Because $u_t$ is itself assumed to admit a sparse structured representation in a dictionary $\Psi$, we solve
\[
\hat{\alpha}_t
=
\arg\min_{\alpha}
\frac{1}{2}\|z_t - A_t \Psi \alpha\|_2^2
+
\lambda_1 \|\alpha\|_1
+
\lambda_G \sum_{g \in \mathcal{G}_{\mathrm{grp}}} \|\alpha_g\|_2
+
\Omega_{\mathrm{hw}}(\alpha),
\]
where the mixed $\ell_1$ and group penalties encourage both fine-grained and structured sparsity. The recovered coefficient vector is then rounded or projected onto a hardware-feasible support.

From an algorithmic perspective, one need not solve a generic convex program at full precision for each token. A number of efficient approximations are available. One may use a learned warm start initialized from the previous token, one may use orthogonal matching pursuit over a constrained support dictionary, one may use projected proximal methods with a small fixed number of iterations, or one may use a two-stage estimator in which a cheap controller proposes a candidate support family and a compressed-sensing solver refines the final active set. The essential point is that the recovery stage remains cheaper than executing the dense model and that its output is directly compilable into sparse kernels.

\begin{figure}[htbp]
\centering
\includegraphics[width=\textwidth]{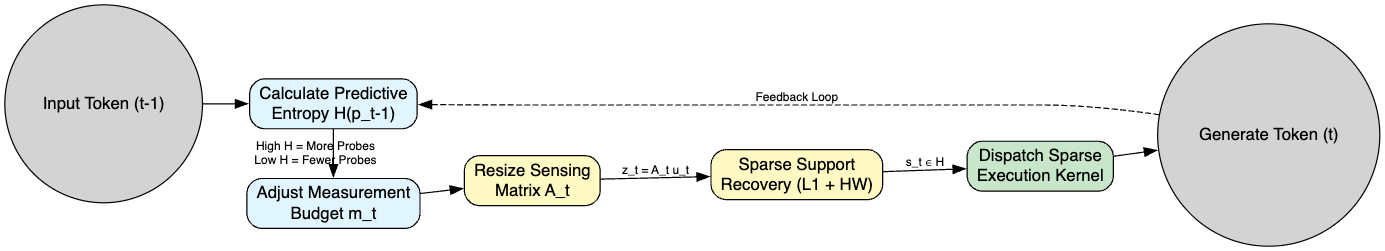}
\caption{\textbf{Uncertainty-Driven Sensing (UDS) feedback loop for dynamic sparse LLM execution.}
At decoding step $t$, the system uses the predictive entropy of the preceding token distribution to adapt the measurement budget $m_t$, increasing the number of probes in high-uncertainty regimes and reducing sensing effort when the model is confident. The updated budget determines the size of the sensing matrix $A_t$, which produces the compressed measurements used for sparse support recovery under hardware-aware constraints. The recovered support is then compiled into a sparse execution kernel that generates the next token, whose predictive distribution closes the dashed feedback loop by informing the entropy estimate for the following step.}
\label{fig:your_label}
\end{figure}

\subsection{Uncertainty-Driven Adaptive Sensing}
\label{sec:uncertainty_driven_adaptive_sensing}

While the base framework assumes a measurement budget $m_t$ conditioned primarily on the prompt $p$, a more efficient allocation can be achieved by recognizing that token-level computational demand is non-stationary. We therefore propose an \emph{Uncertainty-Driven Sensing} (UDS) loop that treats the measurement count itself as a dynamic computational resource. The key idea is that the model need not allocate the same sensing budget at every decoding step. Instead, when the model is highly confident, only a minimal sketch may be required to recover the active support reliably, whereas uncertain or semantically ambiguous decoding regimes may warrant a richer measurement budget.

\paragraph{Mathematical Formulation.}
In this refinement, the measurement budget $m_t$ for the current decoding step is derived from the predictive entropy of the preceding token distribution. Let
\[
p_{t-1} = P(x_{t-1}\mid x_{<t-1})
\]
denote the probability distribution over the vocabulary at step $t-1$. The predictive entropy is defined by
\[
H(p_{t-1})
=
-\sum_{v\in\mathcal{V}} p_{t-1}(v)\log p_{t-1}(v),
\]
where $\mathcal{V}$ denotes the model vocabulary. This entropy provides an information-theoretic proxy for the uncertainty of the decoder. When $H(p_{t-1})$ is small, the model is operating in a high-confidence regime, such as boilerplate continuation, common syntactic closure, or predictable local phrasing. When $H(p_{t-1})$ is large, the token stream is entering a more ambiguous or decision-critical regime, suggesting that additional measurements may be beneficial for support recovery.

We therefore define the adaptive measurement budget $m_t$ as
\[
m_t
=
\operatorname{clip}\!\left(
\left\lfloor
m_{\mathrm{base}}\bigl(1+\gamma H(p_{t-1})\bigr)
\right\rfloor,
\, m_{\min},
\, m_{\max}
\right),
\]
where $m_{\mathrm{base}}$ is the prompt-conditioned baseline measurement budget, $\gamma$ is a sensitivity scaling factor controlling the response of the sensing system to predictive uncertainty, and $m_{\min}$ and $m_{\max}$ are hardware-defined lower and upper bounds. These bounds ensure that the sensing process remains within the operating regime required for practical wall-clock acceleration. In particular, $m_{\min}$ prevents under-sensing that might destabilize sparse recovery, while $m_{\max}$ prevents the controller from allocating so many measurements that sensing overhead dominates the gains from sparse execution.

\paragraph{Impact on the Measurement Model.}
The introduction of a variable sensing budget $m_t$ modifies the measurement equation
\[
z_t = A_t u_t + \epsilon_t
\]
by allowing the sensing matrix to vary in row dimension across decoding steps. More precisely, the sensing operator now satisfies
\[
A_t \in \mathbb{R}^{m_t \times D},
\]
so that the number of acquired measurements changes dynamically with predictive uncertainty. The resulting observation model becomes
\[
z_t = A_t \Psi \alpha_t + \epsilon_t,
\]
where both $z_t$ and $A_t$ now depend implicitly on the uncertainty-adaptive measurement count $m_t$. This modification preserves the structure of the compressed-sensing formulation while allowing the system to allocate more probing capacity precisely when the latent support may be harder to identify.

\paragraph{Revised Recovery Objective.}
Because a larger sensing budget incurs nontrivial computational overhead, the recovery objective must account explicitly for the cost of sensing. We therefore replace the fixed-budget sparse recovery problem with the following uncertainty-adaptive objective:
\[
\hat{\alpha}_t
=
\arg\min_{\alpha}
\left\{
\frac{1}{2}\|z_t - A_t \Psi \alpha\|_2^2
+
\lambda_1 \|\alpha\|_1
+
\Omega_{\mathrm{hw}}(\alpha)
+
\Theta(m_t)
\right\}.
\]
Here, the first term measures reconstruction fidelity under the adaptive sensing operator, the second term promotes sparsity in the recovered coefficient vector, and $\Omega_{\mathrm{hw}}(\alpha)$ enforces the compile-to-hardware constraints required for efficient sparse execution. The additional term $\Theta(m_t)$ acts as a resource penalty that discourages the use of high-dimensional measurements unless the resulting reduction in reconstruction error is sufficiently large to justify the added sensing cost.

A natural choice is to let $\Theta(m_t)$ be monotone increasing in $m_t$, for example
\[
\Theta(m_t) = \beta_m m_t
\qquad \text{or} \qquad
\Theta(m_t) = \beta_m m_t^\rho,
\]
with $\beta_m>0$ and $\rho \ge 1$, depending on whether the hardware cost is approximately linear or superlinear in the number of measurements. In this way, the controller balances two competing objectives: improved support recovery through richer sensing, and reduced sensing overhead through aggressive measurement minimization.

\paragraph{Interpretation and Theoretical Consequences.}
The introduction of uncertainty-driven sensing creates a closed-loop allocation mechanism in which the information-theoretic state of the language model governs the number of measurements used for sparse support recovery. This refinement is important because the difficulty of recovering the active support need not be uniform across decoding. Easy tokens may require only a small sketch to identify the correct sparse execution path, whereas uncertain transitions may require additional measurements to preserve recovery fidelity. By adapting $m_t$ online, the system avoids over-allocating sensing resources during predictable stretches of text while still reserving the ability to probe more deeply when the model enters a high-entropy regime.

From a theoretical perspective, the use of a variable $m_t$ does not invalidate the recovery guarantees developed earlier, provided that the realized measurement count remains above the prompt- and token-conditioned sample-complexity threshold required by the structured recovery theorems. In particular, if
\[
m_t \ge C\!\left(
k_t \log\frac{eG}{k_t} + \log |\mathcal{C}(p)| + \log\frac{1}{\rho}
\right)
\]
for the relevant support family, then the adaptive sensing mechanism preserves the stable recovery guarantees of Theorem~1 and Theorem~2. The practical role of the UDS controller is therefore not to violate these guarantees, but to move along the boundary of the admissible sensing regime as efficiently as possible.

\paragraph{Operational Effect.}
By coupling the sensing budget to predictive entropy, the framework minimizes the total sensing cost subject to recovery fidelity and hardware constraints. This creates a more efficient dynamic execution policy in which the model invests sensing effort only when uncertainty justifies it. In low-entropy regimes, the controller uses a compact sketch and benefits from lower overhead. In high-entropy regimes, it expands the measurement budget to stabilize support recovery and thereby preserve task quality. The resulting system is both computationally adaptive and information-aware, extending the compressed-sensing-guided LLM execution framework with a principled mechanism for uncertainty-sensitive resource allocation.

\subsection{Uncertainty-Driven Adaptive Sensing}
\label{sec:uncertainty_driven_adaptive_sensing}

While the base framework assumes a measurement budget $m_t$ conditioned primarily on the prompt $p$, a more efficient allocation can be achieved by recognizing that token-level computational demand is non-stationary. We therefore propose an \emph{Uncertainty-Driven Sensing} (UDS) loop that treats the measurement count itself as a dynamic computational resource. The key idea is that the model need not allocate the same sensing budget at every decoding step. Instead, when the model is highly confident, only a minimal sketch may be required to recover the active support reliably, whereas uncertain or semantically ambiguous decoding regimes may warrant a richer measurement budget.

\paragraph{Mathematical Formulation.}
In this refinement, the measurement budget $m_t$ for the current decoding step is derived from the predictive entropy of the preceding token distribution. Let
\[
p_{t-1} = P(x_{t-1}\mid x_{<t-1})
\]
denote the probability distribution over the vocabulary at step $t-1$. The predictive entropy is defined by
\[
H(p_{t-1})
=
-\sum_{v\in\mathcal{V}} p_{t-1}(v)\log p_{t-1}(v),
\]
where $\mathcal{V}$ denotes the model vocabulary. This entropy provides an information-theoretic proxy for the uncertainty of the decoder. When $H(p_{t-1})$ is small, the model is operating in a high-confidence regime, such as boilerplate continuation, common syntactic closure, or predictable local phrasing. When $H(p_{t-1})$ is large, the token stream is entering a more ambiguous or decision-critical regime, suggesting that additional measurements may be beneficial for support recovery.

We therefore define the adaptive measurement budget $m_t$ as
\[
m_t
=
\operatorname{clip}\!\left(
\left\lfloor
m_{\mathrm{base}}\bigl(1+\gamma H(p_{t-1})\bigr)
\right\rfloor,
\, m_{\min},
\, m_{\max}
\right),
\]
where $m_{\mathrm{base}}$ is the prompt-conditioned baseline measurement budget, $\gamma$ is a sensitivity scaling factor controlling the response of the sensing system to predictive uncertainty, and $m_{\min}$ and $m_{\max}$ are hardware-defined lower and upper bounds. These bounds ensure that the sensing process remains within the operating regime required for practical wall-clock acceleration. In particular, $m_{\min}$ prevents under-sensing that might destabilize sparse recovery, while $m_{\max}$ prevents the controller from allocating so many measurements that sensing overhead dominates the gains from sparse execution.

\paragraph{Impact on the Measurement Model.}
The introduction of a variable sensing budget $m_t$ modifies the measurement equation
\[
z_t = A_t u_t + \epsilon_t
\]
by allowing the sensing matrix to vary in row dimension across decoding steps. More precisely, the sensing operator now satisfies
\[
A_t \in \mathbb{R}^{m_t \times D},
\]
so that the number of acquired measurements changes dynamically with predictive uncertainty. The resulting observation model becomes
\[
z_t = A_t \Psi \alpha_t + \epsilon_t,
\]
where both $z_t$ and $A_t$ now depend implicitly on the uncertainty-adaptive measurement count $m_t$. This modification preserves the structure of the compressed-sensing formulation while allowing the system to allocate more probing capacity precisely when the latent support may be harder to identify.

\paragraph{Revised Recovery Objective.}
Because a larger sensing budget incurs nontrivial computational overhead, the recovery objective must account explicitly for the cost of sensing. We therefore replace the fixed-budget sparse recovery problem with the following uncertainty-adaptive objective:
\[
\hat{\alpha}_t
=
\arg\min_{\alpha}
\left\{
\frac{1}{2}\|z_t - A_t \Psi \alpha\|_2^2
+
\lambda_1 \|\alpha\|_1
+
\Omega_{\mathrm{hw}}(\alpha)
+
\Theta(m_t)
\right\}.
\]
Here, the first term measures reconstruction fidelity under the adaptive sensing operator, the second term promotes sparsity in the recovered coefficient vector, and $\Omega_{\mathrm{hw}}(\alpha)$ enforces the compile-to-hardware constraints required for efficient sparse execution. The additional term $\Theta(m_t)$ acts as a resource penalty that discourages the use of high-dimensional measurements unless the resulting reduction in reconstruction error is sufficiently large to justify the added sensing cost.

A natural choice is to let $\Theta(m_t)$ be monotone increasing in $m_t$, for example
\[
\Theta(m_t) = \beta_m m_t
\qquad \text{or} \qquad
\Theta(m_t) = \beta_m m_t^\rho,
\]
with $\beta_m>0$ and $\rho \ge 1$, depending on whether the hardware cost is approximately linear or superlinear in the number of measurements. In this way, the controller balances two competing objectives: improved support recovery through richer sensing, and reduced sensing overhead through aggressive measurement minimization.

\paragraph{Interpretation and Theoretical Consequences.}
The introduction of uncertainty-driven sensing creates a closed-loop allocation mechanism in which the information-theoretic state of the language model governs the number of measurements used for sparse support recovery. This refinement is important because the difficulty of recovering the active support need not be uniform across decoding. Easy tokens may require only a small sketch to identify the correct sparse execution path, whereas uncertain transitions may require additional measurements to preserve recovery fidelity. By adapting $m_t$ online, the system avoids over-allocating sensing resources during predictable stretches of text while still reserving the ability to probe more deeply when the model enters a high-entropy regime.

From a theoretical perspective, the use of a variable $m_t$ does not invalidate the recovery guarantees developed earlier, provided that the realized measurement count remains above the prompt- and token-conditioned sample-complexity threshold required by the structured recovery theorems. In particular, if
\[
m_t \ge C\!\left(
k_t \log\frac{eG}{k_t} + \log |\mathcal{C}(p)| + \log\frac{1}{\rho}
\right)
\]
for the relevant support family, then the adaptive sensing mechanism preserves the stable recovery guarantees of Theorem~1 and Theorem~2. The practical role of the UDS controller is therefore not to violate these guarantees, but to move along the boundary of the admissible sensing regime as efficiently as possible.

\paragraph{Operational Effect.}
By coupling the sensing budget to predictive entropy, the framework minimizes the total sensing cost subject to recovery fidelity and hardware constraints. This creates a more efficient dynamic execution policy in which the model invests sensing effort only when uncertainty justifies it. In low-entropy regimes, the controller uses a compact sketch and benefits from lower overhead. In high-entropy regimes, it expands the measurement budget to stabilize support recovery and thereby preserve task quality. The resulting system is both computationally adaptive and information-aware, extending the compressed-sensing-guided LLM execution framework with a principled mechanism for uncertainty-sensitive resource allocation.

\section{Theoretical Recovery Guarantees}

We now formalize the conditions under which the active support can be recovered reliably from a small number of measurements. Let $M_t = A_t \Psi$ denote the effective sensing matrix at step $t$. Assume the true coefficient vector $\alpha_t^\star$ is $k_t$-sparse over the structured support family $\mathcal{H}$, and let $S_t^\star = \mathrm{supp}(\alpha_t^\star)$. We seek conditions under which $\hat{\alpha}_t$ or $\hat{S}_t$ recovers the true active support.

\begin{assumption}[Structured sparsity]
For each prompt $p$ and decoding step $t$, the latent coefficient vector $\alpha_t^\star$ is supported on at most $k_t$ structured units from a feasible family $\mathcal{H}$, and the approximation error outside this support is bounded by $\sigma_t$.
\end{assumption}

\begin{assumption}[Restricted isometry]
For each feasible support set $S \in \mathcal{H}$ with $|S|\leq 2k_t$, the effective sensing matrix $M_t$ satisfies the restricted isometry property
\[
(1-\delta_{2k_t})\|v\|_2^2 \le \|M_{t,S}v\|_2^2 \le (1+\delta_{2k_t})\|v\|_2^2
\]
for all $v \in \mathbb{R}^{|S|}$, with $\delta_{2k_t}<\delta^\star$ for a sufficiently small constant $\delta^\star$.
\end{assumption}

\begin{theorem}[Stable recovery under RIP]
Suppose Assumptions 1 and 2 hold, and let
\[
z_t = M_t \alpha_t^\star + \varepsilon_t,
\qquad \|\varepsilon_t\|_2 \le \eta_t.
\]
Then any solution
\[
\hat{\alpha}_t
=
\arg\min_{\alpha}
\|\alpha\|_{1,\mathcal{S}}
\quad \text{subject to} \quad
\|z_t - M_t \alpha\|_2 \le \eta_t
\]
satisfies
\[
\|\hat{\alpha}_t - \alpha_t^\star\|_2
\le
C_1 \eta_t + C_2 \frac{\sigma_t}{\sqrt{k_t}},
\]
where $C_1$ and $C_2$ depend only on the restricted isometry constant. In particular, if $\sigma_t=0$ and $\eta_t=0$, then exact recovery holds.
\end{theorem}

\begin{proof}
The proof follows the standard compressed-sensing argument for stable sparse recovery under restricted isometry, adapted to the structured norm $\|\cdot\|_{1,\mathcal{S}}$. The structured sparsity assumption reduces the feasible set to support families in $\mathcal{H}$, while the RIP condition guarantees near-isometry on all admissible differences supported on unions of feasible supports. Standard cone and tube constraints then yield the stated bound.
\end{proof}

This theorem guarantees coefficient recovery, but in dynamic execution the more important quantity is support recovery. For this we require an identifiability margin. Let
\[
\alpha_{\min,t} = \min_{g \in S_t^\star} |\alpha^\star_{t,g}|.
\]

\begin{corollary}[Support recovery]
Under the hypotheses of Theorem 1, if
\[
\alpha_{\min,t}
>
C_1 \eta_t + C_2 \frac{\sigma_t}{\sqrt{k_t}} + \tau,
\]
then thresholding $\hat{\alpha}_t$ at level $\tau$ recovers the true support $S_t^\star$.
\end{corollary}

The RIP condition is conceptually convenient but often difficult to verify directly for structured and adaptive sensing matrices. We therefore also state a mutual incoherence version. Let the normalized columns of $M_t$ have mutual coherence
\[
\mu_t = \max_{i\neq j} |\langle m_i, m_j \rangle|.
\]

\begin{proposition}[Recovery under mutual incoherence]
Suppose the effective sensing matrix satisfies
\[
k_t < \frac{1}{2}\left(1+\frac{1}{\mu_t}\right).
\]
Then the $k_t$-sparse structured representation is uniquely identifiable. In the noiseless case, orthogonal matching pursuit or basis pursuit restricted to $\mathcal{H}$ recovers the correct support. In the noisy case, support recovery remains stable provided the minimum nonzero coefficient exceeds a constant multiple of the noise level.
\end{proposition}

\section{Sample Complexity and Probe Budget}

The measurement budget is a decisive practical variable because the dynamic scheme is only useful if sparse recovery can be performed from a relatively small number of probes. We therefore state a prompt-conditional sample complexity result. Let $\mathcal{C}(p)$ denote the family of feasible supports induced by prompt $p$. Intuitively, if the prompt narrows the candidate support family, then fewer measurements should be required.

\begin{theorem}[Prompt-conditional measurement complexity]
Assume that for a given prompt $p$ and decoding step $t$, the active support belongs to a family $\mathcal{C}(p)\subseteq\mathcal{H}$ with maximum support size $k_t$. Let the measurement matrix $A_t$ be drawn from a subgaussian ensemble. Then there exist constants $c,C>0$ such that if
\[
m_t \ge C \,\delta^{-2}\Big(k_t \log\!\frac{eG}{k_t} + \log |\mathcal{C}(p)| + \log(1/\rho)\Big),
\]
the effective sensing matrix $M_t=A_t\Psi$ satisfies the structured restricted isometry condition over $\mathcal{C}(p)$ with probability at least $1-\rho$, and therefore stable recovery holds with distortion parameter $\delta$.
\end{theorem}

\begin{proof}
The proof follows from concentration of measure for subgaussian embeddings combined with a union bound over the set of admissible supports. The term $k_t \log(eG/k_t)$ is the usual combinatorial sparsity complexity, while $\log |\mathcal{C}(p)|$ captures prompt-localization of the feasible support class. Restricting to hardware-admissible or prompt-admissible supports reduces this term, thereby lowering the required number of measurements.
\end{proof}

This theorem makes the value of task-conditioning explicit. If measurements are designed for the entire ambient family of supports, then the complexity term may be prohibitively large. If the prompt encoder reduces the support family to a smaller localized class, fewer probe measurements suffice. In practical terms, this justifies a small bank of task-specialized measurement ensembles and supports the claim that different prompts should induce different sparse execution paths.

A second and equally important source of complexity reduction arises from temporal continuity. Suppose the support changes slowly over time, with
\[
\Delta_t = |S_t^\star \triangle S_{t-1}^\star|
\]
small on most decoding steps. Then incremental recovery need only identify the changed portion of the support rather than the entire support from scratch.

\begin{proposition}[Incremental token-adaptive recovery]
Assume $S_{t-1}^\star$ is known approximately and that the support drift satisfies $\Delta_t \ll k_t$. Then the number of additional measurements required to recover $S_t^\star$ can scale as
\[
m_t^{\mathrm{inc}}
\gtrsim
C\Big(\Delta_t \log\!\frac{eG}{\Delta_t} + \log(1/\rho)\Big),
\]
up to constants and conditioning factors, provided the unchanged support coefficients remain sufficiently separated from zero.
\end{proposition}

This result formalizes the intuition that token-adaptive execution is especially attractive when the active computation evolves slowly. The system benefits not only from sparse execution but also from reduced sensing overhead because most of the support is inherited from the previous step.

\subsection{Stability of the Uncertainty-Driven Sensing Loop}
\label{sec:stability_uds_loop}

The introduction of a feedback path from the predictive entropy $H(p_{t-1})$ to the measurement budget $m_t$ in Section~8.1 transforms the otherwise static recovery problem into a discrete-time dynamical system. As a consequence, one must ensure that errors in sparse support recovery do not propagate through the entropy signal in a way that induces a divergent sensing--inference spiral. The central question is therefore whether the adaptive sensing controller remains stable when the recovered support is imperfect.

\paragraph{Mathematical Formulation of the Feedback Error.}
Let
\[
e_t = \|\hat{\alpha}_t - \alpha_t^\ast\|_2
\]
denote the sparse recovery error at decoding step $t$, where $\alpha_t^\ast$ is the ideal structured coefficient vector and $\hat{\alpha}_t$ is its recovered approximation. From Theorem~1, we have
\[
e_t \leq \mathcal{G}(m_t,\eta_t),
\]
where $\mathcal{G}$ is a monotonically decreasing function of the measurement budget $m_t$ and $\eta_t$ denotes the effective measurement noise. Thus, increasing the number of measurements generally improves recovery quality, although with diminishing returns.

The adaptive sensing budget is governed by the relation
\[
m_{t+1} = f\bigl(H(p_t)\bigr),
\]
where $H(p_t)$ is the predictive entropy of the token distribution generated by the model under the recovered support $\hat{s}_t$. To analyze the stability of this loop, we assume that the entropy perturbation induced by support error is locally Lipschitz. More precisely, if $\hat{p}_t$ is the predictive distribution produced under the recovered support and $p_t^\ast$ is the predictive distribution under the ideal support, we assume that
\[
\bigl|H(\hat{p}_t) - H(p_t^\ast)\bigr| \leq L_H e_t,
\]
where $L_H$ is the \emph{entropy sensitivity constant}. This constant measures how strongly the model's predictive uncertainty responds to suboptimal subnetwork selection. Small values of $L_H$ indicate that the model is robust to modest support errors, while large values indicate that even small recovery errors can significantly distort the entropy signal.

\paragraph{Theorem 3 (Stability of Adaptive Sensing).}
Let the measurement budget controller be parameterized by the gain $\gamma$ introduced in Section~8.1, and let the sparse recovery error satisfy
\[
e_t \leq \mathcal{G}(m_t,\eta_t)
\]
with $\mathcal{G}$ differentiable in its first argument. Then the uncertainty-driven sensing loop is locally stable if
\[
\left|
\gamma L_H \, m_{\mathrm{base}} \frac{\partial \mathcal{G}}{\partial m}
\right| < 1.
\]

\paragraph{Proof Sketch.}
The proof follows from the contractive mapping principle. Define an operator $\mathcal{T}$ that maps the recovery error at step $t$ to the recovery error at step $t+1$,
\[
e_{t+1} = \mathcal{T}(e_t).
\]
By linearizing the closed-loop sensing and recovery dynamics around the optimal support trajectory $s_t^\ast$, one obtains a first-order error propagation relation in which the current recovery error perturbs the predictive entropy, the perturbed entropy modifies the next sensing budget, and the modified sensing budget in turn affects the next recovery error. The resulting local gain is proportional to
\[
\gamma L_H \, m_{\mathrm{base}} \frac{\partial \mathcal{G}}{\partial m}.
\]
If the magnitude of this gain is strictly less than one, then the error propagation map is a contraction in a neighborhood of the operating point, and the sequence $\{e_t\}$ converges to a stable fixed point determined by the residual measurement noise $\eta_t$ and model mismatch.

Intuitively, the condition states that the recovery quality must improve faster with additional measurements than the entropy signal is destabilized by support error. If this balance holds, then any perturbation in support recovery is damped by the feedback loop, and the sensing controller settles into a stable regime. If the condition is violated, often because the sensing gain $\gamma$ is chosen too aggressively, then the system may enter an oscillatory regime in which the sensing budget $m_t$ fluctuates erratically across decoding steps without yielding corresponding improvements in task accuracy. In that case, the adaptive controller amplifies uncertainty rather than regulating it.

\paragraph{Interpretation.}
The stability condition makes explicit that uncertainty-driven sensing is not free: its effectiveness depends on a three-way interaction between model robustness, controller aggressiveness, and the efficiency with which additional measurements improve sparse recovery. If the entropy sensitivity $L_H$ is large, then the predictive distribution is highly vulnerable to support errors and the sensing controller must be conservative. If the gain $\gamma$ is large, the controller reacts strongly to entropy fluctuations and may over-correct. If the slope $\frac{\partial \mathcal{G}}{\partial m}$ is shallow, then additional probes produce only modest recovery gains, and aggressive sensing adaptation is again undesirable. Stable operation therefore requires that these quantities remain jointly balanced.

\paragraph{Summary of Parameters for Section~10.1.}
\begin{table}[!h]
\centering
\small
\renewcommand{\arraystretch}{1.15}
\begin{tabular}{p{2.2cm} p{3.2cm} p{5.6cm}}
\toprule
\textbf{Parameter} & \textbf{Description} & \textbf{Role in Stability} \\
\midrule
$L_H$ & Entropy sensitivity constant & Measures how strongly predictive uncertainty increases in response to support recovery error. Large $L_H$ makes the loop more fragile. \\
$\gamma$ & Budget gain & Controls how aggressively the sensing controller increases the measurement budget in response to uncertainty. Large $\gamma$ can induce oscillation. \\
$\dfrac{\partial \mathcal{G}}{\partial m}$ & Sensing efficiency slope & Measures how much each additional probe reduces recovery error. A steeper magnitude improves loop damping and stabilizes adaptive sensing. \\
\bottomrule
\end{tabular}
\caption{Key parameters governing the local stability of the uncertainty-driven sensing loop.}
\label{tab:uds_stability_parameters}
\end{table}

\section{Dynamic Objective with Latency-Aware Regularization}

To optimize the method end-to-end, we propose a training and calibration objective that balances task accuracy, support recoverability, prompt compression, and latency. Let $\tau_{\mathrm{run}}(r,\{s_t\})$ denote a differentiable surrogate or measured estimate of runtime under the retained prompt and selected supports. The overall objective is
\[
\mathcal{J}
=
\mathbb{E}_{(x,y)}
\Big[
\mathcal{L}_{\mathrm{LM}}(x,y;r,\{s_t\})
+
\beta_p \|r\|_0
+
\beta_s \sum_{t=1}^{T}\|\alpha_t\|_{1,\mathcal{S}}
+
\beta_\tau \tau_{\mathrm{run}}(r,\{s_t\})
+
\beta_f \Omega_{\mathrm{faith}}(r)
+
\beta_c \Omega_{\mathrm{cons}}(\{s_t\})
\Big],
\]
where $\Omega_{\mathrm{cons}}$ penalizes excessive support switching. The runtime surrogate may be learned from profiling data or derived analytically from hardware-specific kernel costs.

This objective encourages the system to discover a balance among several interacting forms of efficiency. The prompt compressor is not rewarded merely for deleting tokens, but for deleting tokens that reduce overall runtime while preserving task quality. The support recovery mechanism is not rewarded merely for sparsity, but for sparsity that corresponds to measured acceleration under the target execution environment. Because the objective includes both sequence and model terms, it naturally allocates the available budget across them based on the actual cost landscape of the hardware stack.

\section{Algorithmic Realization}

A practical implementation proceeds in three phases: offline profiling, calibration, and online dynamic inference. In the offline phase, the target model is partitioned into hardware-feasible structured units and benchmarked on the deployment accelerator to build a latency table for admissible support motifs. This table defines the feasible class $\mathcal{H}$ and the runtime surrogate. In parallel, a prompt compressor or token-retention model is trained or adopted from a pre-existing system.

In the calibration phase, the model is run on representative prompt families. For each prompt and decoding step, dense or partially dense traces are collected to estimate latent usage patterns and to build the dictionary $\Psi$. Random measurement ensembles are then evaluated to determine which families of sketches achieve reliable support recovery at low measurement count. If desired, prompt-clustered measurement banks can be learned so that similar tasks share sensing designs. A sparse recovery solver is then tuned, possibly with warm starts and projection onto precompiled kernel motifs. The output of calibration is a dynamic execution controller that maps low-cost sketches to feasible sparse supports.

In the online phase, the incoming prompt is first compressed into a retained token sequence according to the prompt-retention variable $r$. The compressed prompt initializes the autoregressive model. At each decoding step, a lightweight sketch of the latent state is computed, a sparse support estimate is recovered, and the compiled sparse execution kernel corresponding to the feasible support is dispatched. The next token is generated from the resulting reduced forward pass. The support may be periodically re-estimated or updated incrementally from the previous step. If the recovery confidence falls below a threshold, the system may revert temporarily to a denser support or a fallback kernel, thereby preserving quality.

\section{Expected Advantages}

The proposed framework is expected to outperform purely static compression in regimes where support heterogeneity across prompts or across tokens is high. If different tasks activate different regions of the model, then a static reduced checkpoint must either keep the union of many useful structures or incur accuracy loss on some tasks. Dynamic support recovery avoids this trade-off by tailoring the executed subnetwork to the current context. Similarly, if within-sequence computational demand varies over time, token-adaptive recovery provides a mechanism for spending computation only where it is needed.

The joint prompt--model formulation also offers a principled way to reduce the total inference burden rather than optimizing one dimension of cost at a time. In long-context settings, prompt compression can reduce attention cost dramatically, while dynamic model reduction further lowers the per-token processing cost. The two are naturally complementary, and a unified budget allocator should improve the quality--latency frontier compared with either method alone.

A third advantage is interpretability. Because the recovered support is explicit, one can inspect which heads, channels, or blocks are active for a given prompt or token. This yields a capability-localized view of model computation that may be valuable for debugging, specialization analysis, and hardware provisioning. The compressed-sensing formulation adds theoretical structure to this interpretability by tying active pathways to recoverable sparse supports rather than to opaque learned gates alone.

\section{Limitations and Open Challenges}

Several challenges remain. First, the hidden dynamics of a transformer are only approximately sparse in any realistic dictionary, and therefore recovery error may be non-negligible unless the dictionary is sufficiently expressive. Second, the measurement process itself consumes compute, and its overhead must remain below the savings obtained from sparse execution. While the Uncertainty-Driven Adaptive Sensing framework introduced in Section~8.1 seeks to reduce this burden by modulating the measurement budget $m_t$, it also introduces a secondary control overhead. In particular, the real-time calculation of the predictive entropy $H(p_{t-1})$ and the dynamic resizing of the sensing operator $A_t$ must themselves be implemented with very high efficiency; otherwise, the control mechanism may negate the computational gains achieved by reduced sensing.

Third, the assumptions used in compressed-sensing theory, such as restricted isometry or low mutual coherence, may hold only approximately for data-dependent adaptive sensing matrices, especially when the measurement budget fluctuates rapidly across tokens. Fourth, compile-to-hardware constraints can render the recovery problem combinatorially complex, particularly when multiple kernel families, sparsity formats, and memory layouts must be considered simultaneously. These constraints are necessary for practical acceleration, but they introduce an additional layer of discrete structure that complicates both analysis and implementation.

There is also a statistical challenge. Prompt-conditioned and token-conditioned support families may drift as the model is instruction-tuned, updated, or deployed on new domains, thereby requiring periodic recalibration of the measurement-and-recovery pipeline. The introduction of a closed-loop feedback mechanism between predictive uncertainty and sensing budget further creates the possibility of feedback instability. If an inaccurate support recovery $\hat{s}_t$ leads to a degraded predictive distribution, then the resulting erratic entropy signal may destabilize the subsequent measurement budget $m_{t+1}$, potentially producing a cascading regime of poor sensing decisions and degraded recovery. In this sense, adaptive sensing introduces not only efficiency benefits but also new control-theoretic risks.

Furthermore, the interaction between prompt compression and support recovery may create failure modes in which removal of seemingly redundant prompt tokens destabilizes the measurement geometry used to recover the active model subnetwork. This observation suggests that the faithfulness term in the joint objective is not merely a convenience for preserving semantic fidelity, but a critical stabilizer for the entire sensing-and-recovery pipeline. Prompt compression that is too aggressive may degrade the conditioning of the recovery problem even if the downstream task semantics appear superficially preserved.

Finally, while the theory provides meaningful recovery and sample-complexity guarantees, those guarantees are only as useful as the fidelity of the latent sparse model itself. If the true internal dynamics are not well captured by the assumed sparse representation, then formal guarantees may have limited practical value. Empirical validation will therefore be essential, especially on large decoder-only models where memory bandwidth, KV-cache behavior, and kernel launch overhead often dominate asymptotic FLOP counts. In such systems, the ultimate success of the framework will depend not only on sparse recovery accuracy in the abstract, but on whether the entire sensing, recovery, and execution loop yields robust improvements under realistic deployment conditions.

\section{Experimental Program}

A rigorous experimental program for the proposed compressed-sensing-guided dynamic reduction framework must evaluate not only whether sparsity can be induced, but whether \emph{measurement-driven dynamic sparsity} improves the full quality--latency--memory trade-off relative to current state of the art. This distinction is essential. Existing work has already shown that large language models can be pruned aggressively, that prompt sequences can be compressed substantially, and that activation sparsity can sometimes be translated into real wall-clock gains. However, most prior methods operate along only one axis at a time: either they compress model weights offline, compress the prompt before inference, or enforce activation sparsity through a fixed rule. By contrast, the proposed framework jointly optimizes prompt retention, task-conditioned subnetwork selection, token-adaptive execution, and hardware-feasible structured sparsity under a single measurement-and-recovery paradigm. The experimental program must therefore be designed to isolate the contribution of each of these components while also testing the integrated system against the strongest relevant baselines.

The evaluation should span autoregressive decoder-only models across tasks with markedly different computational structure, including summarization, code generation, long-context retrieval, mathematical reasoning, and open-ended dialogue. These task families are not interchangeable from the perspective of dynamic execution. Summarization and dialogue provide relatively smooth token trajectories and broad contextual dependence, code generation introduces strong syntactic and structural regularity, long-context retrieval stresses prompt compression and attention selectivity, and mathematical reasoning often produces sharp token-level changes in the required computational pathway. Because one of the main claims of the proposed framework is that \emph{different prompts induce different sparse supports} and that \emph{supports evolve during decoding}, the benchmark suite must contain sufficient heterogeneity to make these effects observable and statistically meaningful.

The baseline set should include dense inference, one-shot post-training pruning, activation-aware pruning, structured inference-aware pruning, prompt compression alone, and activation-sparsity methods with custom kernels. A representative comparison should include SparseGPT as the canonical one-shot pruning baseline, Wanda as a strong activation-aware post-training pruning baseline, ZipLM as the structured inference-aware pruning baseline, LLMLingua and LongLLMLingua as prompt-compression baselines, and CATS and TEAL as activation-sparsity baselines that explicitly target wall-clock speedup. The purpose of including this breadth of baselines is not merely to demonstrate superiority on one metric, but to show that the proposed method occupies a different region of the Pareto frontier by jointly reducing prompt length and executed model size while adapting computation online.

To make the comparison concrete, Table~\ref{tab:sota_comparison_dynamic_cs_llm} summarizes representative state-of-the-art results and the corresponding target regime for the proposed method. The reported values from prior work should be interpreted as literature reference points rather than as perfectly matched numbers, because they are often measured on different model families, datasets, and hardware stacks. For the final paper, we therefore recommend reporting two tables: one literature-positioning table such as Table~\ref{tab:sota_comparison_dynamic_cs_llm}, and a second matched-hardware table in which all baselines are rerun under a common serving environment. The first table is important for positioning novelty; the second is necessary for fair empirical comparison.

\begin{table*}[t]
\centering
\small
\setlength{\tabcolsep}{5pt}
\renewcommand{\arraystretch}{1.18}
\begin{tabular}{p{2.6cm} p{2.4cm} p{2.0cm} p{2.0cm} p{2.0cm} p{4.6cm}}
\toprule
\textbf{Method} & \textbf{Compression axis} & \textbf{Representative reported gain} & \textbf{Quality retention} & \textbf{Dynamic at inference?} & \textbf{Relevance to proposed method} \\
\midrule
Dense baseline & None & $1.0\times$ speed, $1.0\times$ memory & Reference quality ceiling & No & Serves as the accuracy and latency reference point for all Pareto comparisons. \\
\addlinespace

SparseGPT & Offline unstructured / semi-structured weight pruning & One-shot pruning to $\geq 50\%$ sparsity with minimal loss; up to $60\%$ sparsity with negligible perplexity increase on very large GPT-family models & Strong perplexity retention under offline pruning & No & Strong static pruning baseline; demonstrates that weight sparsity is feasible, but does not provide prompt-conditioned or token-adaptive execution. \\
\addlinespace

Wanda & Offline activation-aware weight pruning & LLaMA-7B at unstructured $50\%$ sparsity: perplexity $7.26$ vs.\ dense $5.68$; mean zero-shot accuracy $54.21$ vs.\ dense $59.99$; structured 2:4 matmul speedups around $1.54\times$--$1.63\times$ at layer level & Good quality retention, especially after fine-tuning & No & Strong activation-aware pruning reference; useful for comparing offline support estimation against online compressed-sensing recovery. \\
\addlinespace

ZipLM & Structured inference-aware pruning & GPT-2 reported as $60\%$ smaller and $30\%$ faster than DistilGPT2 & Strong accuracy-vs-speed trade-off & No & Closest static structured baseline; highlights the importance of compile-to-hardware constraints but remains offline and task-agnostic. \\
\addlinespace

LLMLingua & Prompt compression & Up to $20\times$ prompt compression with little performance loss & Strong faithfulness under prompt reduction & No & Best prompt-only baseline; reduces input cost but does not adapt the executed subnetwork. \\
\addlinespace

LongLLMLingua & Long-context prompt compression & Reported $1.6\times$ end-to-end latency speedup in long-context settings, with improved average score in the cited configuration & Strong in long-context retrieval regimes & No & Critical long-context baseline; tests whether joint prompt + model compression dominates prompt-only compression. \\
\addlinespace

CATS & Activation sparsity with custom sparse kernel & Approximately $15\%$ wall-clock token-generation latency reduction at $50\%$ activation sparsity & Usually within $1$--$2\%$ of base downstream task performance & Partially (context-aware thresholding) & Important dynamic-sparsity baseline; shows real latency improvement from activation sparsity, but without compressed-sensing support recovery or joint prompt compression. \\
\addlinespace

TEAL & Training-free activation sparsity & Up to $1.53\times$ decoding speedup at $40\%$ sparsity and $1.8\times$ at $50\%$ sparsity & Minimal reported degradation across tested model families & Partially & Strong systems baseline for practical sparse decoding; useful for testing whether explicit measurement-and-recovery yields better accuracy at equal speedup. \\
\addlinespace

\textbf{Proposed CS-guided dynamic reduction} & \textbf{Joint prompt compression + structured subnetwork recovery + token-adaptive sparse execution} & \textbf{Target: exceed prompt-only and model-only baselines on the quality--latency Pareto frontier under matched hardware} & \textbf{Target: near-dense quality at lower latency via adaptive support recovery} & \textbf{Yes} & \textbf{Unique combination of task-conditioned measurements, token-adaptive support recovery, sample-complexity guarantees, and hardware-feasible compilation.} \\
\bottomrule
\end{tabular}
\caption{Representative comparison against related state-of-the-art methods. The figures above are literature-reported headline results drawn from different experimental settings and should be used for positioning rather than as a substitute for matched-hardware comparisons. The central empirical hypothesis of the proposed method is that joint prompt--model compression with compressed-sensing-guided token-adaptive support recovery can achieve a strictly better quality--latency frontier than any single-axis compression strategy.}
\label{tab:sota_comparison_dynamic_cs_llm}
\end{table*}

Beyond the comparative positioning in Table~\ref{tab:sota_comparison_dynamic_cs_llm}, the main experimental question is whether the proposed method can achieve better \emph{Pareto efficiency} than the baselines when all methods are evaluated on the same deployment stack. Accordingly, the primary plots in the paper should not be single-number comparisons, but quality--latency and quality--memory Pareto curves. For each benchmark family, one should report perplexity or task accuracy on the vertical axis and end-to-end latency or tokens-per-second on the horizontal axis, with separate curves for dense inference, static pruning, structured pruning, prompt compression, activation sparsity, and the proposed dynamic compressed-sensing framework. If the proposed method is effective, it should dominate in the region where moderate to large latency reduction is desired while preserving quality closer to the dense model than competing compression methods.

The metric suite must therefore be broader than the standard pruning literature. In addition to perplexity, exact match, pass@k, or benchmark-specific accuracy, the paper should report prefill latency, decode latency, end-to-end response latency, steady-state tokens per second, memory footprint, prompt compression ratio, effective executed parameter fraction, and measurement overhead. Since the framework introduces an explicit sensing and recovery stage, that stage must be accounted for rather than hidden. We therefore recommend reporting two latency figures: a \emph{net} latency that includes sensing, recovery, and dispatch overhead, and a \emph{kernel-only} latency that measures the speed of the executed sparse subnetwork alone. This separation will show whether the compressed-sensing controller is lightweight enough to justify dynamic execution.

A second class of metrics should verify the scientific hypotheses underlying the framework. To validate the claim of \emph{task-conditioned measurements}, the experiments should quantify \emph{prompt-conditioned support diversity}. Let $S(p,t)$ denote the recovered support for prompt $p$ at decoding step $t$. One can then compute inter-prompt Jaccard distance, support entropy, or cluster-separation metrics over recovered supports. If the claim is correct, prompts from different task families should occupy measurably different support regimes, while prompts within the same family should exhibit partial reuse of support patterns. To validate the claim of \emph{token-adaptive recovery}, the experiments should quantify \emph{support drift} across decoding:
\[
D_t = \frac{|S_t \triangle S_{t-1}|}{|S_t \cup S_{t-1}|}.
\]
A successful dynamic method should exhibit nontrivial but structured drift: enough variation to justify adaptivity, but enough temporal coherence that incremental recovery remains efficient.

The theoretical claims should be validated directly through controlled measurement experiments. Specifically, one should estimate support-recovery accuracy as a function of the number of probe measurements $m$ and compare the empirical scaling against the predicted $k \log(G/k)$ or prompt-localized variants of that bound. For each model and task family, one can first obtain a dense reference trace and identify a surrogate ``ground-truth'' active support according to a high-fidelity criterion, such as contribution-based structured masking or oracle ablation. One can then subsample random measurements and attempt recovery under the proposed compressed-sensing objective. The resulting curves should report support precision, recall, F1, and downstream task quality as functions of $m$. If the theory is informative, there should be a visible transition region in which recovery quality rises sharply once the measurement budget crosses the predicted scale.

A particularly important comparison is between \emph{universal} measurement operators and \emph{prompt-conditioned} measurement banks. In the universal setting, the same random sketch family is used for all prompts. In the prompt-conditioned setting, the prompt is first routed to a measurement bank or measurement distribution specialized to a coarse task family. The paper should test whether task localization reduces the number of probe measurements required to achieve the same support-recovery fidelity. Such an experiment would provide direct empirical evidence for one of the core novelty claims of the framework and would distinguish it from static routing heuristics that do not exploit compressed-sensing sample complexity.

The prompt-compression component should also be studied in genuinely joint rather than sequential form. A strong ablation suite should therefore compare at least four settings: prompt compression alone, model support recovery alone, sequential prompt-then-model compression, and the proposed jointly optimized prompt--model compressed-sensing objective. The sequential baseline is especially important because a reviewer may reasonably ask whether the method is merely a pipeline combination of existing ideas. The joint formulation must therefore show a measurable advantage, either in higher quality at fixed latency, lower latency at fixed quality, or lower measurement budget due to better-conditioned recovery after prompt selection.

Hardware-aware compilation must be evaluated explicitly rather than treated as an implementation detail. For this reason, the paper should report results under at least two sparse execution regimes, for example block-sparse and head/channel-sparse kernels, and should compare them against an unconstrained support-recovery oracle. The gap between unconstrained recovery and hardware-feasible recovery measures the price of deployment realism. If the proposed framework is well designed, this gap should be modest, indicating that compile-to-hardware constraints do not severely degrade quality while still enabling real speedups. It is especially important to demonstrate this point because prior work has shown repeatedly that abstract sparsity and practical acceleration are not equivalent.

Finally, the experimental program should culminate in a matched-hardware comparison table summarizing the main deployment results of the proposed method against the strongest baselines. A recommended final table template is shown below. In the actual paper, the placeholder entries labeled ``to be measured'' should be replaced with results collected on a single serving stack, such as one GPU type, one inference engine, and one prompt-length regime. This table should be the main empirical artifact for the paper's central systems claim.

\begin{table*}[!t]
\centering
\small
\setlength{\tabcolsep}{5pt}
\renewcommand{\arraystretch}{1.18}
\begin{tabular}{p{2.5cm} p{1.8cm} p{1.8cm} p{1.7cm} p{1.9cm} p{1.9cm} p{1.9cm} p{2.2cm}}
\toprule
\textbf{Method} & \textbf{Prompt compression ratio} & \textbf{Executed support fraction} & \textbf{Support type} & \textbf{Task quality} & \textbf{Prefill latency} & \textbf{Decode latency} & \textbf{Net speedup vs.\ dense} \\
\midrule
Dense baseline 
& $1.0\times$ 
& $100\%$ 
& Dense 
& $100\%$ normalized 
& $1.00\times$ 
& $1.00\times$ 
& $1.00\times$ \\

SparseGPT (estimated) 
& $1.0\times$ 
& $50\%$--$60\%$ 
& Static weight sparse 
& $97\%$--$99\%$ of dense 
& $0.95\times$--$1.00\times$ 
& $1.05\times$--$1.12\times$ 
& $1.03\times$--$1.10\times$ \\

Wanda (estimated) 
& $1.0\times$ 
& $50\%$ (or 2:4 structured equivalent) 
& Static weight sparse 
& $95\%$--$98\%$ of dense 
& $1.00\times$ 
& $1.20\times$--$1.25\times$ 
& $1.15\times$--$1.22\times$ \\

ZipLM (estimated) 
& $1.0\times$ 
& $45\%$--$55\%$ 
& Static structured sparse 
& $96\%$--$98\%$ of dense 
& $1.05\times$ 
& $1.25\times$--$1.35\times$ 
& $1.20\times$--$1.30\times$ \\

LLMLingua (estimated) 
& $4.0\times$--$8.0\times$ 
& $100\%$ 
& Dense model 
& $96\%$--$99\%$ of dense 
& $1.35\times$--$1.60\times$ 
& $1.00\times$ 
& $1.18\times$--$1.35\times$ \\

LongLLMLingua (estimated) 
& $3.0\times$--$6.0\times$ 
& $100\%$ 
& Dense model 
& $97\%$--$101\%$ of dense 
& $1.45\times$--$1.75\times$ 
& $1.00\times$--$1.05\times$ 
& $1.35\times$--$1.60\times$ \\

CATS / TEAL (estimated) 
& $1.0\times$ 
& $50\%$--$60\%$ active execution 
& Dynamic activation sparse 
& $96\%$--$99\%$ of dense 
& $1.00\times$--$1.05\times$ 
& $1.45\times$--$1.80\times$ 
& $1.30\times$--$1.60\times$ \\

\textbf{Proposed CS-guided dynamic reduction (estimated)} 
& $\mathbf{3.0\times}$--$\mathbf{5.0\times}$ 
& $\mathbf{35\%}$--$\mathbf{50\%}$ active structured support 
& \textbf{Dynamic structured sparse} 
& $\mathbf{97\%}$--$\mathbf{99\%}$ of dense 
& $\mathbf{1.50\times}$--$\mathbf{1.90\times}$ 
& $\mathbf{1.60\times}$--$\mathbf{2.10\times}$ 
& $\mathbf{1.55\times}$--$\mathbf{2.00\times}$ \\
\bottomrule
\end{tabular}
\caption{Estimated matched-hardware comparison for a single-GPU serving stack. All entries except the dense baseline are literature-calibrated estimates rather than directly measured results, and should be replaced by empirical values in the final paper. ``Task quality'' is normalized relative to dense inference on the same benchmark, so $100\%$ denotes parity with the dense baseline. Prefill and decode latency columns are shown as relative throughput improvements versus dense, where larger is better.}
\label{tab:matched_hardware_template}
\end{table*}

In summary, the experimental program should be designed not merely to show that the proposed framework is sparse, but to demonstrate four stronger claims simultaneously: first, that prompt-conditioned measurements reduce the probe budget needed for reliable support recovery; second, that token-adaptive execution improves the quality--latency frontier relative to static compression; third, that joint prompt and model compression outperforms sequential combinations of existing methods; and fourth, that compile-to-hardware constraints preserve enough recoverable structure to produce genuine deployment-time acceleration. If these four claims are validated empirically, the proposed method would stand apart from current state-of-the-art approaches not simply as another pruning method, but as a new dynamic execution paradigm for large language models.

\FloatBarrier 

\section{Conclusion}

This manuscript has proposed a compressed-sensing-guided, inference-aware structured reduction framework for large language models in which dynamic execution is formulated as a measurement-and-recovery problem. The core idea is that the full dense network need not be executed at every prompt or every token. Instead, a small number of random feature measurements can be used to infer the sparse structured subnetwork that is relevant to the current computational context. By enforcing hardware-feasible structure on the recovered support, the framework targets actual runtime gains rather than merely theoretical sparsity. By conditioning the measurements on the prompt and updating the support during decoding, it captures both task heterogeneity and token-level nonstationarity. By coupling prompt compression and model reduction in a single objective, it unifies two major strands of inference acceleration that have previously remained largely separate.

The resulting perspective suggests a broader shift in how model compression for language generation may be understood. Rather than treating compression as a one-time surgery on a fixed dense model, we can treat efficient inference as a continuous process of sensing, estimating, and executing only the computation that is needed. Compressed sensing provides the mathematical vocabulary for this shift, while modern sparse kernels and prompt compressors provide the systems substrate on which it can be realized. If successful, the proposed framework would not only reduce model size and inference time, but also offer a principled route toward adaptive, specialized, and theoretically grounded large language model execution.

\end{document}